\documentclass{article} 
\usepackage{times}
\usepackage{graphicx} 
\usepackage{subfigure} 

\usepackage{natbib}

\usepackage{algorithm}
\usepackage{algorithmic}

\usepackage{hyperref}

\usepackage[accepted]{icml2016}

\usepackage{url}
\usepackage{amssymb}
\usepackage{amsmath}
\usepackage{amsthm}
\usepackage{chngcntr}

\usepackage{color}
\usepackage{lscape}
\usepackage[figuresleft]{rotating}
\usepackage{sidecap}
\usepackage{booktabs, topcapt}

\newcommand{\para}[1]{{\bf #1}\quad}
\newcommand{\be}{\begin{eqnarray} \begin{aligned}}
\newcommand{\ee}{\end{aligned} \end{eqnarray} }
\newcommand{\benn}{\begin{eqnarray*} \begin{aligned}}
\newcommand{\eenn}{\end{aligned} \end{eqnarray*} }

\makeatletter
\def\thm@space@setup{%
  \thm@preskip=\parskip \thm@postskip=0pt
}
\makeatother

\newtheorem{theorem}{Theorem}[section]
\newtheorem*{theorem*}{Theorem}
\newtheorem{lemma}[theorem]{Lemma}

\newtheorem{corollary}[theorem]{Corollary}

\newenvironment{definition}[1][Definition]{\begin{trivlist}
\item[\hskip \labelsep {\bfseries #1}]}{\end{trivlist}}

\usepackage{titlesec}

\icmltitlerunning{The Information Sieve}

\begin{document}

\twocolumn[
\icmltitle{The Information Sieve}

\icmlauthor{Greg Ver Steeg}{gregv@isi.edu}
\icmladdress{University of Southern California, Information Sciences Institute,
            Marina del Rey, CA 90292 USA}
\icmlauthor{Aram Galstyan}{galstyan@isi.edu}
\icmladdress{University of Southern California, Information Sciences Institute,
            Marina del Rey, CA 90292 USA}

\icmlkeywords{machine learning, information theory, representation learning, unsupervised learning}

\vskip 0.3in
]
\begin{abstract}
We introduce a new framework for unsupervised learning of representations based on a novel hierarchical decomposition of information. 
Intuitively, data is passed through a series of progressively fine-grained sieves.
Each layer of the sieve recovers a single latent factor that is maximally informative about multivariate dependence in the data. 
The data is transformed after each pass so that the remaining unexplained information trickles down to the next layer.
Ultimately, we are left with a set of latent factors explaining all the dependence in the original data and remainder information consisting of independent noise.
We present a practical implementation of this framework for discrete variables and apply it to a variety of fundamental tasks in unsupervised learning including independent component analysis, lossy and lossless compression, and predicting missing values in data.
\end{abstract}

The hope of finding a succinct principle that elucidates the brain's information processing abilities has often kindled interest in information-theoretic ideas~\cite{barlow,simoncelli}.  
In machine learning, on the other hand, the past decade has witnessed a shift in focus toward expressive, hierarchical models, with successes driven by increasingly effective ways to leverage labeled data to learn rich models~\cite{schmidhuber,bengioreview}.  
Information-theoretic ideas like the venerable InfoMax principle~\cite{linsker,bell95} can be and are applied in both contexts with empirical success but they do not allow us to quantify the \emph{information value} of adding depth to our representations. 
We introduce a novel incremental and hierarchical decomposition of information and show that it defines a framework for unsupervised learning of deep representations in which the information contribution of each layer can be precisely quantified. Moreover, this scheme automatically determines the structure and depth among hidden units in the representation based only on local learning rules. 

The shift in perspective that enables our information decomposition is to focus on how well the learned representation explains multivariate mutual information in the data (a measure originally introduced as ``total correlation''~\cite{watanabe}).  Intuitively, our approach constructs a hierarchical representation of data by passing it through a sequence of progressively fine-grained sieves. At the first layer of the sieve we learn a factor that explains as much of the dependence in the data as possible. The data is then transformed into the ``remainder information'', which has this dependence extracted. The next layer of the sieve looks for the largest source of dependence in the remainder information, and the cycle repeats. At each step, we obtain a successively tighter upper and lower bound on the multivariate information in the data, with convergence between the bounds obtained when the remaining information consists of nothing but independent factors. Because we end up with independent factors, one can also view this decomposition as a new way to do independent component analysis (ICA)~\cite{comon, ica}. Unlike traditional methods, we do not assume a specific generative model of the data (i.e., that it consists of a linear transformation of independent sources) and we extract independent factors incrementally rather than all at once. The implementation we develop here uses only discrete variables and is therefore most relevant for the challenging problem of ICA with discrete variables, which has applications to compression~\cite{feder}.  

After providing some background in Sec.~\ref{sec:background}, we introduce a new way to iteratively decompose the information in data in Sec.~\ref{sec:decomposition}, and show how to use these decompositions to define a practical and incremental framework for unsupervised representation learning in Sec.~\ref{sec:implement}. We demonstrate the versatility of this framework by applying it first to independent component analysis (Sec.~\ref{sec:ica}). Next, we use the sieve as a lossy compression to perform tasks typically relegated to generative models including in-painting and generating new samples (Sec.~\ref{sec:mnist}). Finally, we cast the sieve as a lossless compression and show that it beats standard compression schemes on a benchmark task (Sec.~\ref{sec:lossless}).

\section{Information-theoretic learning background}\label{sec:background}
Using standard notation~\cite{cover}, capital $X_i$ denotes a random variable taking values in some domain and whose instances are denoted in lowercase, $x_i$. In this paper, the domain of all variables are considered to be discrete and finite. We abbreviate multivariate random variables, $X \equiv X_{1:n} \equiv X_1,\ldots,X_n$, with an associated probability distribution, $p_X(X_1=x_1,\ldots, X_n=x_n)$, which is typically abbreviated to $p(x)$.  We will index different groups of multivariate random variables with superscripts, $X^k$, as defined in Fig.~\ref{fig:sieve}. We let $X^0$ denote the original observed variables and we often omit the superscript in this case for readability. 

Entropy is defined in the usual way as $H(X) \equiv \mathbb E_X [ \log 1/p(x)]$. We use base two logarithms so that the unit of information is bits.
Higher-order entropies can be constructed in various ways from this standard definition. For instance, the mutual information between two groups of random variables, $X$ and $Y$ can be written as the reduction of uncertainty in one variable, given information about the other, $ I(X;Y) = H(X) - H(X|Y)$.

The ``InfoMax'' principle~\cite{linsker,bell95} suggests that for unsupervised learning we should construct $Y$'s to maximize their mutual information with $X$, the data. Despite its intuitive appeal, this approach has several potential problems~(see \cite{icml2014} for one example). Here we focus on the fact that the InfoMax principle is not very useful for characterizing ``deep representations'', even though it is often invoked in this context~\cite{bengio_autoencoders}. This follows directly from the data processing inequality (a similar argument appears in~\cite{tishby_deep}). Namely, if we start with $X$, construct a layer of hidden units $Y^1$ that are a function of $X$, and continue adding layers to a stacked representation so that $X \rightarrow Y^1 \rightarrow Y^2 \ldots Y^k$, then the information that the $Y$'s have about $X$ cannot increase after the first layer, $I(X;Y^{1:k}) = I(X;Y^1)$. From the point of view of mutual information, $Y^1$ is a copy and $Y^2$ is just a copy of a copy. While a coarse-grained copy might be useful, the InfoMax principle does not quantify how or why. 

Instead of looking for a $Y$ that memorizes the data, we shift our perspective to searching for a $Y$ so that the $X_i$'s are as independent as possible conditioned on this $Y$. Essentially, we are trying to reconstruct the latent factors that are the cause of the dependence in $X_i$. To formalize this, we introduce the multivariate mutual information which was first introduced as ``total correlation''~\cite{watanabe}. 
\be\label{eq:tc}
TC(X) &\equiv& D_{KL}\left(p(x) || \prod_{i=1}^n p(x_i)\right) \\
&=& \sum_{i=1}^n H(X_i) - H(X) 
\ee
This quantity reflects the dependence in $X$ and is zero if and only if the $X_i$'s are independent. 
Just as mutual information is the reduction of entropy in $X$ after conditioning on $Y$, we can define the reduction in multivariate information in $X$ after conditioning on $Y$. 
\be\label{eq:tcxy}
TC(X;Y) &\equiv& TC(X) - TC(X|Y) \\
&=& \sum_{i=1}^n I(X_i;Y) - I(X;Y).
\ee
That $TC(X)$ can be hierarchically decomposed in terms of short and long range dependencies was already appreciated by Watanabe~\cite{watanabe} and has been used in applications such as hierarchical clustering~\cite{kraskov_cluster}. This provides a hint about how higher levels of hierarchical representations can be useful: more abstract representations should reflect longer range dependencies in the data. 
Our contribution below is to demonstrate a tractable approach for learning a hierarchy of latent factors, $Y$, that exactly capture the multivariate information in $X$.

\section{Incremental information decomposition}\label{sec:decomposition}

We consider any set of probabilistic functions of some input variables, $X$, to be a ``representation'' of $X$. 
Looking at Fig.~\ref{fig:sieve}(a), we consider a representation with a single learned latent factor, $Y$. Then, we try to save the information in $X$ that is not captured by $Y$ into the ``remainder information'', $\bar X$.  The final result is encapsulated in Cor.~\ref{iterative} which says that we can repeat this procedure iteratively (as in Fig.~\ref{fig:sieve}(b)) and $TC(X)$ decomposes into a sum of non-negative contributions from each $Y_k$. Note that $X^{(k)}$ includes $Y_k$, so that $Y$'s at subsequent layers can depend on latent factors learned at earlier layers.

\begin{figure}[tbp] 
   (a) \includegraphics[width=0.95\columnwidth]{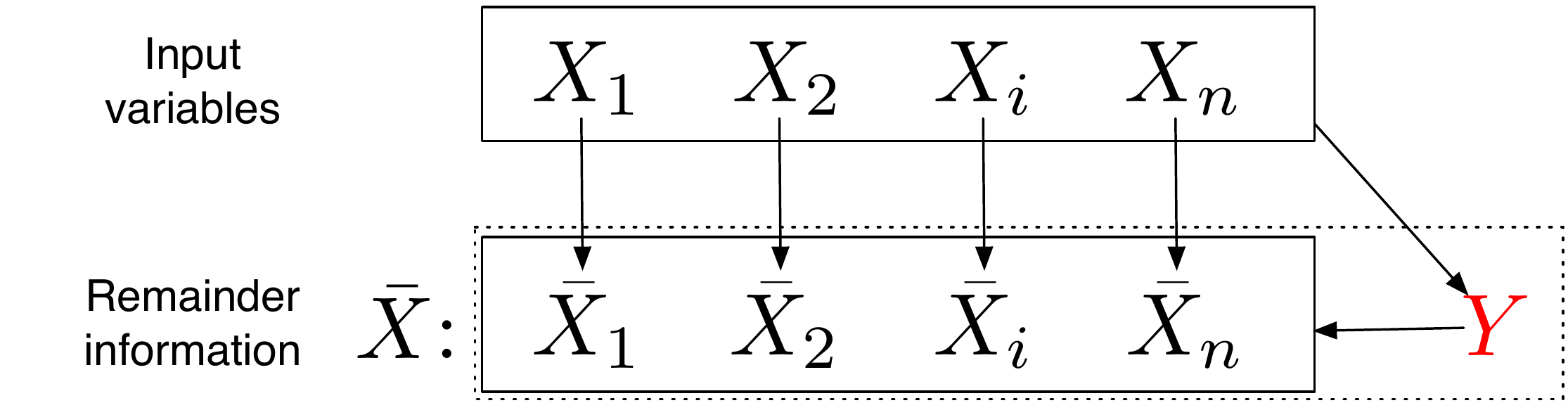} 

   \vspace{4mm}
    (b)
   \begin{minipage}[c]{0.95\columnwidth}
$$
\arraycolsep=1.6pt\def\arraystretch{1.2}
   \begin{array}{rcccccc} 
      X^0: & \bf X_1 & \ldots & \bf X_n & & & \\
      X^1: & X_1^1 & \ldots & X_n^1 & \color{red} Y_1 & & \\
      X^2: & X_1^2 & \ldots & X_n^2 &Y_1^2 & \color{red} Y_2 & \\
      \cdots \\
      X^k: & X_1^k & \ldots & X_n^k &Y_1^k &Y_2^k & \color{red} Y_k \\
   \end{array}
$$
\end{minipage}
   \caption{(a) This diagram describes one layer of the information sieve. In this graphical model, the variables in the top layer ($X_i$'s) represent (observed) input variables. $Y$ is some function of all the $X_i$'s that is optimized to be maximally informative about multivariate dependence in $X$. The remainder information, $\bar X_i$ depends on $X_i$ and $Y$ and is set to contain information in $X_i$ that is not captured by $Y$. (b) Summary of variable naming scheme for multiple layers of the sieve. The input variables are in bold and the learned latent factors are in red. 
   }   \vspace{-2mm}
   \label{fig:sieve}
\end{figure}

\begin{theorem} 
{\bf Incremental Decomposition of Information}\quad
\label{incremental}
Let $Y$ be some (deterministic) function of $X_1, \ldots, X_n$ and let $\bar X_i$ be a probabilistic function of $X_i, Y$, for each $i = 1,\ldots,n$. Then the following upper and lower bounds on $TC(X)$ hold:
\be\label{eq:bounds}
A \leq TC(X) - \left(TC(\bar X) + TC(X;Y)\right) \leq B\\
  A = - \sum_{i=1}^n I(\bar X_i ; Y), \quad B = \sum_{i=1}^n H(X_i | \bar X_i, Y)
\ee 
\end{theorem} 
A proof is provided in App.~B. 
Note that the remainder information, $\bar X \equiv \bar X_1, \ldots, \bar X_n, Y$, includes $Y$. 
Bounds on $TC(X)$ also provide bounds on $H(X)$ by using Eq.~\ref{eq:tc}. 
Next, we point out that the remainder information, $\bar X$, can be chosen to make these bounds tight.
%
\begin{lemma}\label{lemma}
{\bf Construction of perfect remainder information}\quad
For discrete, finite random variables $X_i, Y$ drawn from some distribution, $p(X_i, Y)$, it is possible to define another random variable $\bar X_i \sim p(\bar X_i | X_i, Y)$ that satisfies the following two properties:
\benn
&\mbox{(i)} &I(\bar X_i ; Y) = 0 \quad &\mbox{Remainder contains no $Y$ info} \\
&\mbox{(ii)} &H(X_i|\bar X_i, Y) = 0 \quad &\mbox{Original is perfectly recoverable}
\eenn
\end{lemma} 
We give a concrete construction in App.~C. We would like to point out one caveat here. The cardinality of $\bar X_i$ may have to be large to satisfy these equalities. For a fixed number of samples, this may cause difficulties with estimation, as discussed in Sec.~\ref{sec:implement}. 
With perfect remainder information in hand, our decomposition becomes exact. 
%
\begin{corollary}
{\bf Exact decomposition}\quad
\label{exact}
For $Y$ a function of $X$ and perfect remainder information, $\bar X_i, i=1,\ldots,n$, as defined in Lemma~\ref{lemma}, the following decomposition holds:
\be\label{eq:exact}
TC(X) = TC(\bar X) + TC(X;Y)
\ee
\end{corollary} 
The above corollary follows directly from Eq.~\ref{eq:bounds} and the definition of perfect remainder information. Intuitively, it states that the dependence in $X$ can be decomposed into a piece that is explained by $Y$, $TC(X;Y)$, and the remaining dependence in $\bar X$. This decomposition can then be iterated to extract more and more information from the data.
%
\begin{corollary}{\bf Iterative decomposition} \quad
\label{iterative}
Using the variable naming scheme in Fig.~\ref{fig:sieve}(b), we construct a hierarchical representation where each $Y_k$ is a function of $X^{k-1}$ and $X^k$ includes the (perfect) remainder information from $X^{k-1}$ according to Lemma~\ref{lemma}.  
\be\label{eq:iterative}
TC(X) = TC(X^r) +  \sum_{k=1}^r TC(X^{k-1} ; Y_k) \label{eq:decomp}
\ee
\end{corollary} 
It is easy to check that Eq.~\ref{eq:decomp} results from repeated application of Cor.~\ref{exact}. We show in the next section that the quantities of the form $TC(X^{k-1};Y_k)$ can be estimated and optimized over efficiently, despite involving high-dimensional variables. As we add the (non-negative) contributions from optimizing $TC(X^{k-1} ; Y_k)$, the remaining dependence in the remainder information, $TC(X^k)$, must decrease because $TC(X)$ is some data-dependent constant. 
Decomposing data into independent factors is exactly the goal of ICA, and the connections are discussed in Sec.~\ref{sec:ica}.

\section{Implementing the sieve}\label{sec:implement}
Because this learning framework contains many unfamiliar concepts, we consider a detailed analysis of a toy problem in Fig.~\ref{fig:example} while addressing concrete issues in implementing the information sieve. 
\begin{figure}[htbp]
\begin{minipage}[c]{0.45\columnwidth}
\benn 
Y_1 &= \arg \max_{Y=f(X)} TC(X;Y) \\
X_1^1 &= X_1 + Y_1 \mod 2 \\
 X_2^1 &= X_2 + Y_1 \mod 2 \\
 X_3^1 &= X_3
\eenn
\end{minipage}
\quad \quad
   \begin{minipage}[c]{0.4\columnwidth}
$$
\arraycolsep=1.2pt\def\arraystretch{1.}
   \begin{array}{ccc||c|ccc} 
   \multicolumn{3}{c||}{\mbox{Data}} & \multicolumn{4}{c}{\mbox{Remainder}} \\ \hline
      X_1 & X_2 & X_3 & Y_1 & X_1^1 & X_2^1 & X_3^1 \\ \hline
      0 & 0 & 1 & 0 & 0 & 0 & 1   \\
      0 & 0 & 0 & 0 & 0 & 0 & 0   \\
      1 & 1 & 0 & 1 & 0 & 0 & 0   \\
      1 & 1 & 1 & 1 & 0 & 0 & 1  
   \end{array}
$$
\end{minipage}
\caption{A simple example for which we imagine we have samples of $X$ drawn from some distribution.
\vspace{-2mm}
}   \label{fig:example}
\end{figure}

\para{Step 1: Optimizing $TC(X^{k-1};Y_k)$} First, we construct a variable, $Y_k$, that is some arbitrary function of $X^{k-1}$ and that explains as much of the dependence in the data is possible. Note that we have to pick the cardinality of $Y_k$ and we will always use binary variables. 
Dropping the layer indices, $k$, the optimization can be written as follows.
\be\label{eq:opt}
\max_{p(y|x)} \sum_i I(X_i ; Y) - I(X;Y)
\ee
Here, we have relaxed the optimization to allow for probabilistic functions of $X$. If we take the derivative of this expression (along with the constraint that $p(y|x)$ should be normalized)  and set it equal to zero, the following simple fixed point equation emerges.
\benn
p(y|x) = \frac{p(y)}{Z(x)} \prod_{i=1}^n \frac{p(x_i|y)}{p(x_i)}
\eenn
The state space of $X$ is exponentially large in $n$, the number of variables. Fortunately, this fixed point equation tells us that we can write the solution in terms of a linear number of terms which are just marginal likelihood ratios. Details of this optimization are discussed in Sec.~A. Note that the optimization provides a probabilistic function which we round to a deterministic function by taking the most likely value of $Y$ for each $X$. In the example in Fig.~\ref{fig:example}, $TC(X;Y_1) = 1 \mbox{ bit}$, which can be verified from Eq.~\ref{eq:tcxy}. 

Surprisingly, we did not need to restrict or parametrize the set of possible functions; the simple form of the solution was implied by the objective. Furthermore, we can also use this function to find labels for previously unseen examples or to calculate $Y$'s for data with missing variables (details in Sec.~A).  Not only that, but a byproduct of the procedure is to give us a value for the objective $TC(X^{k-1};Y_k)$, which can be estimated even from a small number of samples.  

\para{Step 2: Remainder information} Next, the goal is to construct the remainder information, $X_i^k$, as a probabilistic function of $X_i^{k-1}, Y_k$, so that the following conditions are satisfied: $(i) I(X^k_i ;Y_1)=0$ and  $(ii)  H(X_i^{k-1}|X_i^k,Y_k) =0.$
This can be done exactly and we provide a simple algorithm in Sec.~C. Solutions for this example are given in Fig.~\ref{fig:example}. Concretely, we estimate the marginals, $p(x_i^{k-1}, y_k)$ from data and then write down a conditional probability table, $p(x_i^k | x_i^{k-1}, y_k)$, satisfying the conditions.
The example in Fig.~\ref{fig:example} was constructed so that the remainder information had the same cardinality as the original variables. This is not always possible. While we can always achieve perfect remainder information by letting the cardinality of the remainder information grow, it might become difficult to estimate marginals of the form $p(X^{k-1}_i, Y_k)$ at subsequent layers of the sieve, as is required for the optimization in step 1. In results shown below we allow the cardinality of the variables to increase by only one at each level to avoid state space explosion, even if doing so causes $I(X^{k-1}_i ;Y_k) > 0$. We keep track of these penalty terms so that we can report accurate lower bounds using Eq.~\ref{eq:bounds}.

Another issue to note is that in general there may not be a unique choice for the remainder information. In the example, $I(X_3;Y)=0$ already so we choose $X_3^1 = X_3$, but $X_3^1 = X_3 + Y_1 \mod 2$ would also have been a valid choice. If the identity transformation, $X_i^k = X_i^{k-1}$ satisfies the conditions, we will always choose it. 

\para{Step 3: Repeat until the end} At this point we repeat the procedure, putting the remainder information back into step 1 and searching for a new latent factor that explains any remaining dependency. 
In this case, we can see by inspection that $TC(X^1)=0$ and, using Eq.~\ref{eq:iterative}, we have $TC(X) = TC(X^1) + TC(X;Y_1) = 1\mbox{ bit}$. Generally, in high-dimensional spaces it may be difficult to verify that the remainder information is truly independent.
When the remainder information is independent, the result of attempting the optimization $\max_{p(y_k|x^{k-1})} TC(X^{k-1}; Y_k)=0$. In practice, we stop our hierarchical procedure when the optimization in step 1 stops producing positive results because it means our bounds are no longer tightening. 
Code implementing this entire pipeline is available~\cite{sieve_code}. 

\para{Prediction and compression}
Note that our condition for the remainder information that $H(X_i^{k-1}|X_i^k, Y_k) = 0$ implies that we can perfectly reconstruct each variable $X_i^{k-1}$ from the remainder information at the next layer. Therefore, we can in principle reconstruct the data from the representation at the last layer of the sieve. In the example, the remainder information requires two bits to encode each variable separately, while the data requires three bits to encode each variable separately. The final representation has exploited the redundancy between $X_1, X_2$ to create a more succinct encoding. 
A use case for lossy compression is discussed in Sec.~\ref{sec:mnist}.  
Also note that at each layer some variables are almost or completely explained ($X^1_1,X^1_2$ in the example become constant). Subsequent layers can enjoy a computational speed-up by ignoring these variables that will no longer contribute to the optimization. 

\section{Discrete ICA}\label{sec:ica}

If $X$ represents observed variables then the entropy, $H(X)$, can be interpreted as the average number of bits required to encode a single observation of these variables. In practice, however, if $X$ is high-dimensional then estimating $H(X)$ or constructing this code requires detailed knowledge of $p(x)$, which may require exponentially many samples in the number of variables. 
Going back at least to Barlow~\cite{barlow}, it was recognized that if $X$ is transformed into some other basis, $Y$, with the $Y$'s independent ($TC(Y)=0$), then the coding cost in this new basis is $H(Y) = \sum_j H(Y_j)$, i.e., it is the same as encoding each variable separately. This is exactly the problem of independent component analysis: transform the data into a basis for which $TC(Y)=0$, or is minimized~\cite{comon,ica}. 

\begin{figure*}[tbp] 
   \centering
   \includegraphics[width=0.95\textwidth]{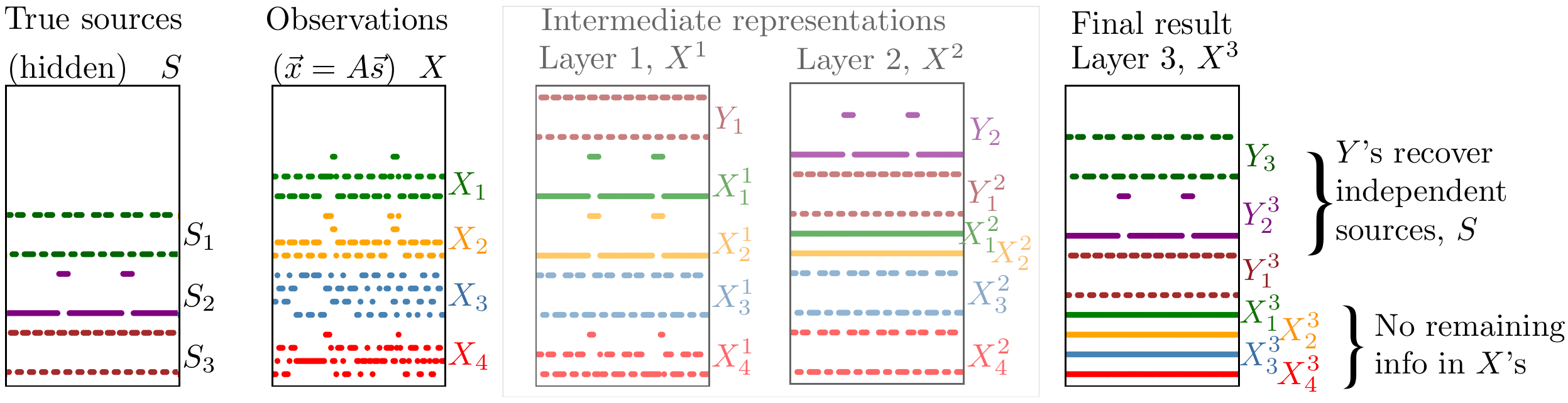} \quad
   \caption{On the far left, we consider three independent binary random variables, $S_1, S_2, S_3$. The vertical position of each signal is offset for visibility.
   From left to right: Independent source data is linearly mixed, $\vec x = A \vec s$.  This data, $X$, is fed into the information sieve. After going through some intermediate representations, the final result is shown on the far right. Consult Fig.~\ref{fig:sieve}(b) for the variable naming scheme. The final representation recovers the independent input sources.  
 }
   \label{fig:ica}   \vspace{-2mm}
\end{figure*}

While our method does not directly minimize the total correlation of $Y$, Eq.~\ref{eq:iterative} shows that, because $TC(X)$ is a data-dependent constant, every increase in the total correlation explained by each latent factor directly implies a reduction in the dependence of the resulting representation, 
$$TC(X^r) = TC(X) - \sum_{k=1}^r TC(X^{k-1} ; Y_k).$$ 
Since the terms in the sum are optimized (and always non-negative), the dependence is decreased at each level. 
That independence could be achieved as a byproduct of efficient coding has been previously considered~\cite{schmidhuberICA}. An approach that leading to ``less dependent components'' for continuous variables has also been shown~\cite{lca}. 

For discrete variables, which are the focus of this paper, performing ICA is a challenging and active area of research. Recent state-of-the-art results lower the complexity of this problem to \emph{only} a single exponential in the number of variables~\cite{feder}. Our method represents a major leap for this problem as it is only linear in the number of variables, however, we only guarantee extraction of components that are \emph{more independent}, while the approach of Painsky et. al. guarantees a global optimum. 

The most commonly studied scenario for ICA is to consider a reconstruction problem where some (typically continuous) and independent source variables are linearly mixed according to some unknown matrix~\cite{comon,ica}. The goal is to recover the matrix and unmix the components (back into their independent sources).  Next we demonstrate our discrete independent component recovery on an example reminiscent of traditional ICA examples. 

\para{An ICA example}
Fig.~\ref{fig:ica} shows an example of recovering independent components from discrete random variables. The sources, $S$, are hidden and the observations, $X$, are a linearly mixture of these sources. The mixing matrix used in this example is  
$$A= ((1, 1, 1),
 (2, 0, -1),
 (1, 2, 0),
 (-1, 1, 0)).$$
 The information sieve continues to add layers as long as it increases the tightness of the information bounds. The intermediate representations at each layers are also shown. For instance, layer 1 extracts one independent component, and then removes this component from the remainder information. After three layers, the sieve stops because $X^3$ consists of independent variables and therefore the optimization of $\max TC(X^3;Y_4) = 0$. 
 
 In this case, the procedure correctly stops after three latent factors are discovered. Naively, three layers makes this a ``deep'' representation. However, we can examine the functional dependence of $Y$'s and $X$'s by looking at the strength of the mutual information, $I(Y_k ; X_i^{k-1})$, as shown in Fig.~\ref{fig:ica_decomp}. This allows us to see that none of the learned latent factors ($Y$'s) depend on each other so the resulting model is actually, in some sense, shallow. The example in the next section, for contrast, has a deep structure where $Y$'s depend on latent factors from previous layers.  Note that the structure in Fig.~\ref{fig:ica_decomp} perfectly reflects the structure of the mixing matrix (i.e., if we flipped the arrows and changes the $Y$'s to $S$'s, this would be an accurate representation of the generative model we used). 
\begin{figure}[htbp] 
   \centering
   \includegraphics[width=0.4\columnwidth]{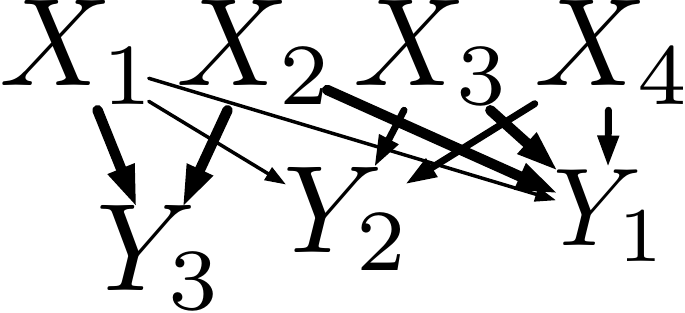} 
   \caption{This visualizes the structure of the learned representation for the ICA example in Fig.~\ref{fig:ica}. The thickness of links is proportional to $I(Y_k ; X_i^{k-1})$. }
   \label{fig:ica_decomp}
     \vspace{-2mm}
\end{figure}
    
While the sieve is guaranteed to recover independent components in some limit, there may be multiple ways to decompose the data into independent components. 
Because our method does not start with the assumption of a linear mixing of independent sources, even if such a decomposition exists we might recover a different one. 
While the example we showed happened to return the linear solution that we used to generate the problem, there is no guarantee to find a linear solution, even if one exists.

\section{Lossy compression on MNIST digits}\label{sec:mnist}

The information sieve is not a generative probabilistic model. We construct latent factors that are functions of the data in a way that maximizes the (multivariate) information that is preserved. Nevertheless, because of the way the remainder information is constructed, we can run the sieve in reverse and, if we throw away the remainder information and keep only the $Y$'s, we get a lossy compression. We can use this lossy compression interpretation to perform tasks that are usually achieved using generative models including in-painting and generating new examples (the converse, interpreting a generative model as lossy compression, has also been considered~\cite{hintonRBM}).

We illustrate the steps for lossy compression and in-painting in Fig.~\ref{fig:lossy}. Imagine that we have already trained a sieve model. For lossy compression, we first transform some data using the sieve. The sieve is an invertible transformation, so we can run it in reverse to exactly recover the inputs. Instead we store only the labels, $Y$, throwing the remainder information, $X_{1:n}^k$, away. When we invert the sieve, what values should we input for $X_{1:n}^k$? During training, we estimate the most likely value to occur for each variable, $X_i^k$. W.l.o.g., we relabel the symbols so that this value is $0$.  Then, for lossy recovery, we run the sieve in reverse using the labels we stored, $Y$, and setting $X_{1:n}^k = 0$. 

In-painting proceeds in essentially the same way. We take advantage of the fact that we can transform data even in the presence of missing values, as described in Sec.~A. Then we replace missing values in the remainder information with $0$'s and invert the sieve normally. 

\begin{figure}[htbp] 
   \centering
  \includegraphics[width=0.95\columnwidth]{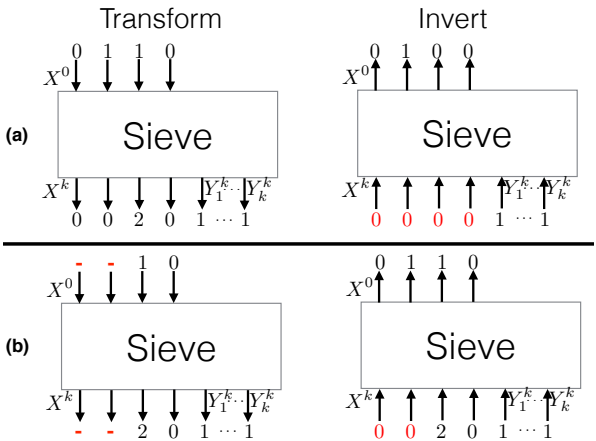} 
   \caption{(a) We use the sieve to transform data into some labels, $Y$ plus remainder information. For lossy recovery, we invert the sieve using only the $Y$'s, setting $X$'s to zero. (b) For in-painting, we first transform data with missing values. Then we invert the sieve, again using zeros for the missing remainder information. }
   \label{fig:lossy}
\end{figure}

\begin{figure}[htbp] 
   \centering
   (a) \includegraphics[width=0.85\columnwidth]{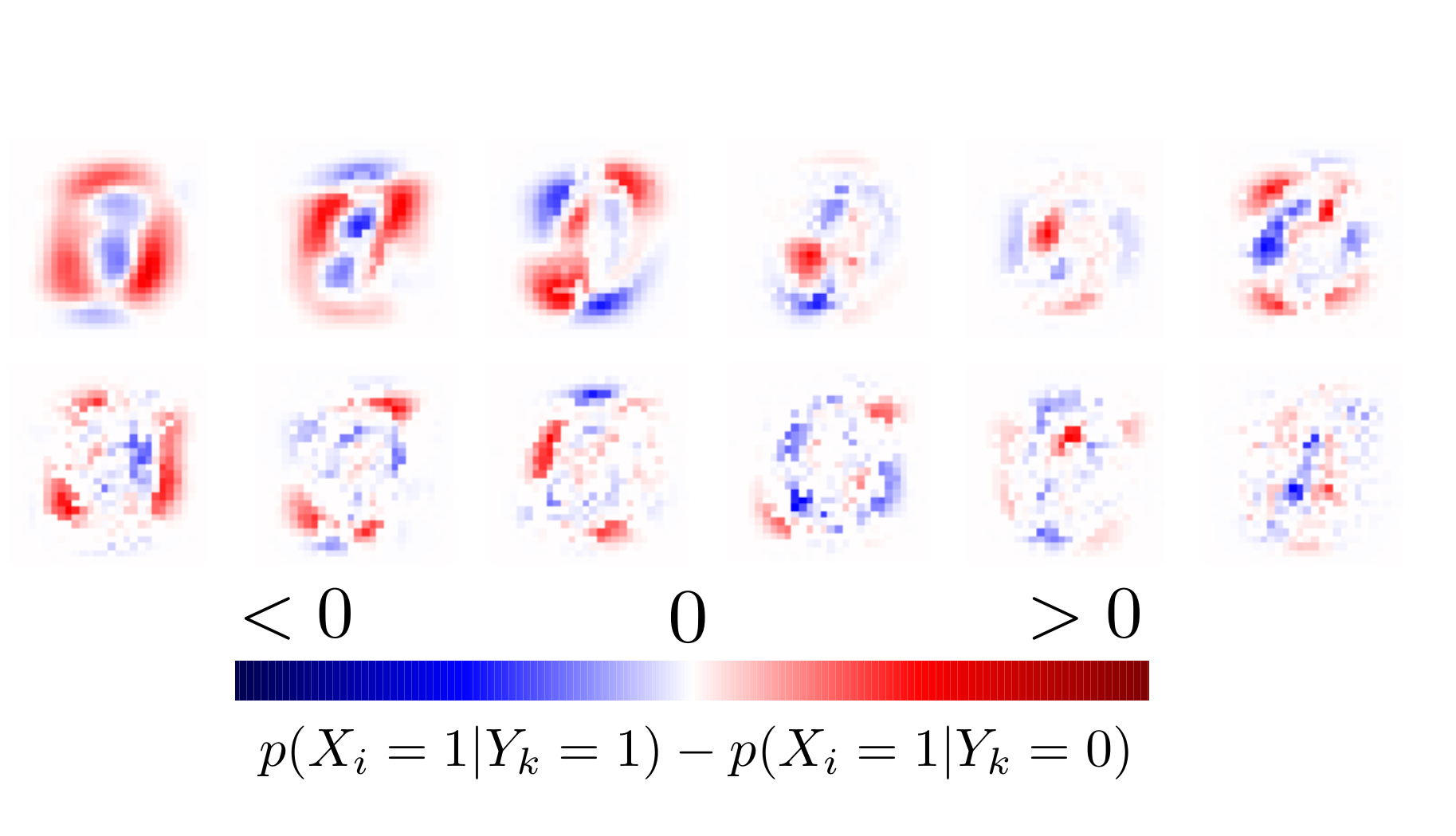}
   
   \vspace{4mm}
   (b)\includegraphics[width=0.6\columnwidth]{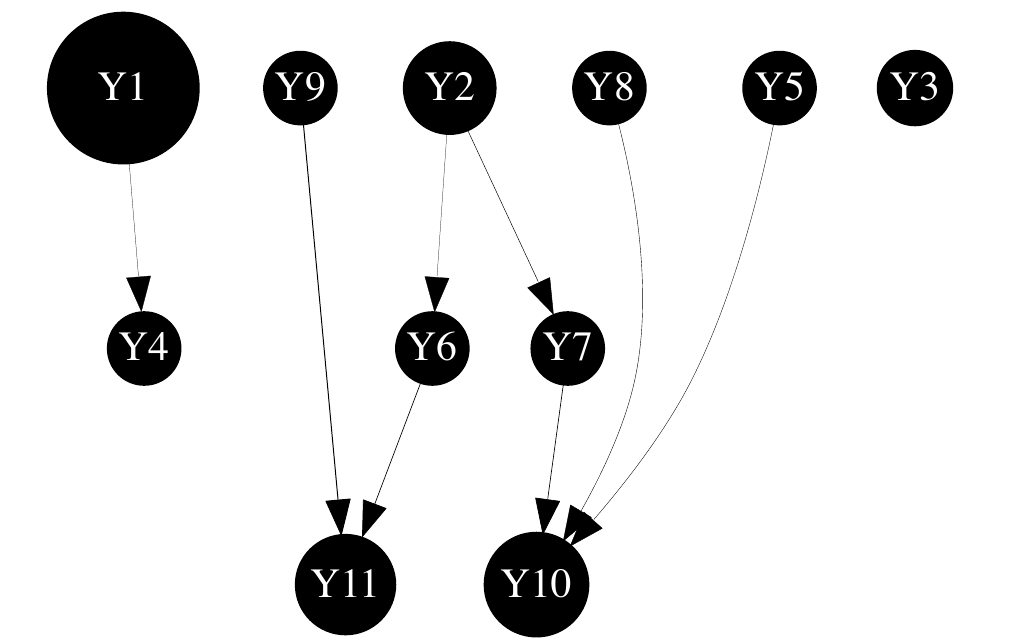}    
   \caption{(a) We visualize each of the learned components, arranged in reading order. (b) The structural relationships among the latent factors is based on $I(Y_k ; Y_j^{k-1})$.  The size of a node represents the magnitude of $TC(X^{k-1};Y_k)$. 
   }    \vspace{-4mm}
   \label{fig:mnist_structure}
\end{figure}

For the following tasks, we consider 50k MNIST digits that were binarized at the normalized grayscale threshold of $0.5$. We include no prior knowledge about spatial structure or invariance under transformations through convolutional structure or pooling, for instance. The $28{\times}28$ binarized images are treated as binary vectors in a 784 dimensional space. The digit labels are also not used in our analysis.
We trained the information sieve on this data, adding layers as long as the bounds were tightening. This led to a 12 layer representation and a lower bound on $TC(X)$ of about 40 bits. It seems likely that more than 12 layers could be effective but the growing size of the state space for the remainder information increases the difficulty of estimation with limited data. 
A visualization of the learned latent factors and the relationships among them appears in Fig.~\ref{fig:mnist_structure}. Unlike the ICA example, the latent factors here have exhibit multi-layered relationships.

\begin{figure*}[htbp] 
   \centering
   \includegraphics[width=0.99\textwidth]{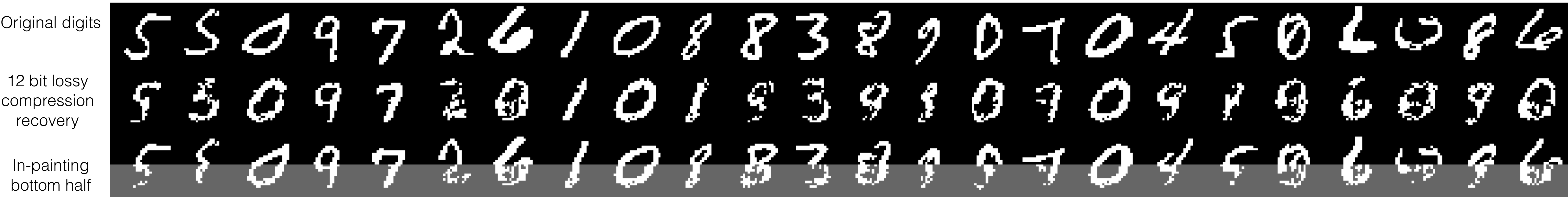} 
   \caption{The top row are randomly selected MNIST digits. In the second row, we compress the digits into the 12 binary variables, $Y_k$, and then attempt to reconstruct the image. In the bottom row, we learn $Y$'s using just the pixels in the top half and then recover the pixels in the bottom half.
    }   \vspace{-2mm}
   \label{fig:digits}
\end{figure*}

The middle row of Fig.~\ref{fig:digits} shows results from the lossy compression task. We use the sieve to transform the original digits into 12 binary latent factors, $Y$, plus remainder information for each pixel, $X^{12}_{1:784}$, and then we use the $Y$'s alone to reconstruct the image. In the third row, the $Y$'s are estimated using only pixels from the top half. Then we reconstruct the pixels on the bottom half from these latent factors. Similar results on test images are shown in Sec.~D, along with examples of ``hallucinating'' new digits.


\section{Lossless compression}\label{sec:lossless}

Given samples of $X$ drawn from $p(x)$, the best compression we can do in theory is to use an average of $H(X)$ bits for our compressed representation~\cite{shannon}.  However, in practice, if $X$ is high-dimensional then we cannot accurately estimate $p(x)$ to achieve this level of compression. We consider alternate schemes and compare them to the information sieve for a compression task. 

\para{Benchmark} 
For a lossless compression benchmark, we consider a set of 60k of binarized digits with 784 pixels, where the order of the pixels has been randomly permuted (the same unknown permutation is applied to each image).
 Note that we have made this task artificially more difficult than the straightforward task of compressing digits because many compression schemes exploit spatial correlations among neighboring pixels for efficiency. 
The information sieve is unaffected by this permutation since it does not make any assumptions about the structure of the input space (e.g. the adjacency of pixels). We use 50k digits as training for models, and report compression results on the 10k test digits. 

Naively, these 28 by 28 binary pixels would require 784 bits per digit to store. However, some pixels are almost always zero. According to Shannon, we can compress pixel $i$ using just $H(X_i)$ bits on average~\cite{shannon}. Because the state space of each individual bit is small, this bound is actually achievable (using arithmetic coding~\cite{cover}, for example). Therefore, we should be able to store the digits using $\sum_i H(X_i) \approx 297 \mbox{ bits/digit}$ on average.   

We would like to make the data more compressible by first transforming it. We consider a simplified version of the sieve with just one layer. We let $Y$ take $m$ possible values and then optimize it according to our objective. For the remainder information, we use the (invertible) function $\bar x_i  = |x_i - \arg \max_{z} p(X_i=z | Y=y)|$. In other words, $\bar X$ represents deviation from the most likely value of $x_i$ for a given value of $y$. The cost of storing a digit in this new representation will be $\log_2 m + \sum_{i=1}^{784} H(\bar X_i)$, where $\log_2 m$ bits are used to store the value of $Y$. 

For comparison, we consider an analogous benchmark introduced in~\cite{compress_mnist}. For this benchmark, we just choose $m$ random digits as representatives (from the training set). Then for each test digit, we store the identity of the closest representative (by Hamming distance), along with the error which we will also call $\bar X_i$, so that we can recover the original digit. Again, the number of bits per digit will just by $\log_2 m$ plus the cost of storing the errors for each pixel according to Shannon. 

\begin{figure}[htbp] 
   \centering
   \includegraphics[width=0.09\columnwidth]{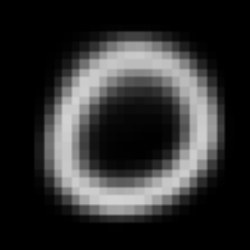}\includegraphics[width=0.09\columnwidth]{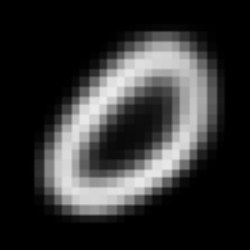}\includegraphics[width=0.09\columnwidth]{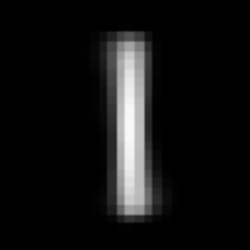}\includegraphics[width=0.09\columnwidth]{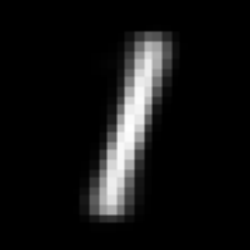}\includegraphics[width=0.09\columnwidth]{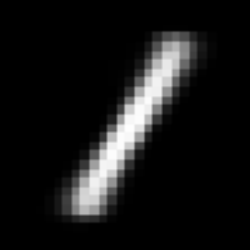}\includegraphics[width=0.09\columnwidth]{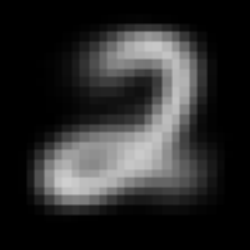}\includegraphics[width=0.09\columnwidth]{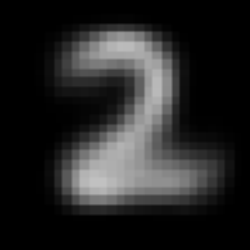}\includegraphics[width=0.09\columnwidth]{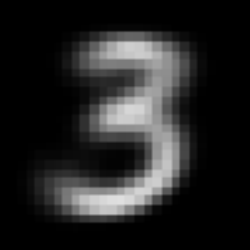}\includegraphics[width=0.09\columnwidth]{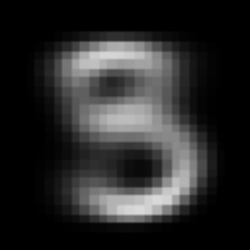}\includegraphics[width=0.09\columnwidth]{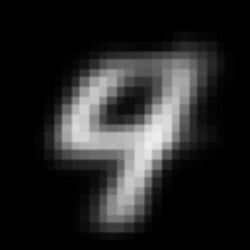}
   
   \includegraphics[width=0.09\columnwidth]{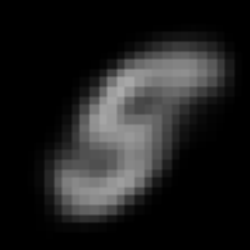}\includegraphics[width=0.09\columnwidth]{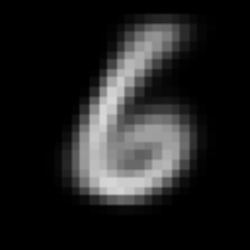}\includegraphics[width=0.09\columnwidth]{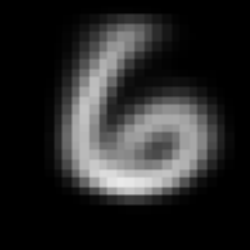}\includegraphics[width=0.09\columnwidth]{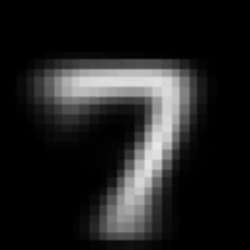}\includegraphics[width=0.09\columnwidth]{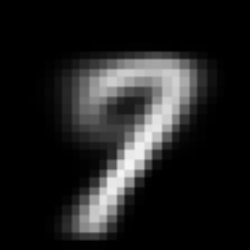}\includegraphics[width=0.09\columnwidth]{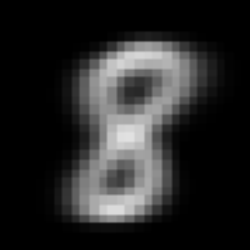}\includegraphics[width=0.09\columnwidth]{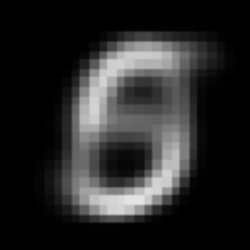}\includegraphics[width=0.09\columnwidth]{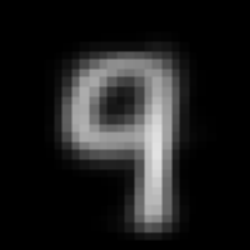}\includegraphics[width=0.09\columnwidth]{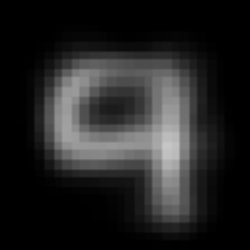}\includegraphics[width=0.09\columnwidth]{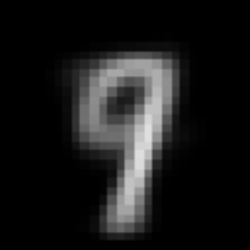}
   \caption{This shows $p(x_i=1|y=k)$ for $k=1,\ldots,20$ for each pixel, $x_i$, in an image.
   }    \vspace{-2mm}
   \label{fig:20components}
\end{figure}
  
Consider the single layer sieve with $Y=1,\ldots, m$ and $m=20$. After optimizing, Fig.~\ref{fig:20components} visualizes the components of $Y$. As an exercise in unsupervised clustering the results are somewhat interesting; the sieve basically finds clusters for each digit and for slanted versions of each digit. 
In Fig.~\ref{fig:20xbar} we explicitly construct the remainder information (bottom row), i.e. the deviation between the most likely value of each pixel conditioned on $Y$ (middle row) and the original (top row). 



\begin{figure}[htbp] 
   \centering
   \includegraphics[width=0.08\columnwidth]{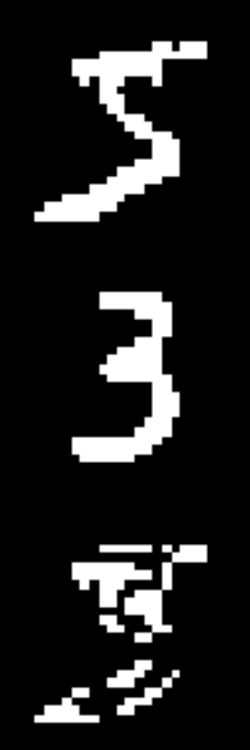}\includegraphics[width=0.08\columnwidth]{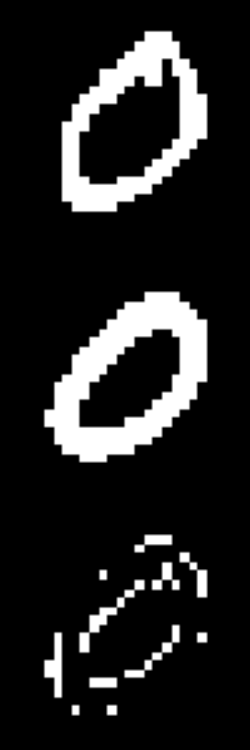}\includegraphics[width=0.08\columnwidth]{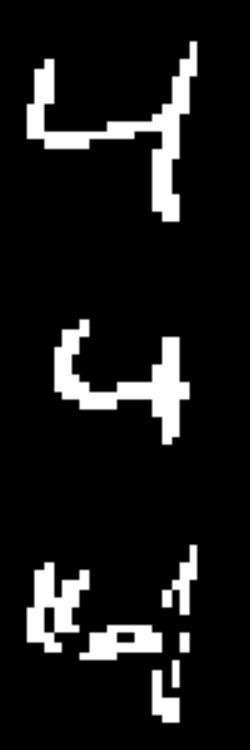}\includegraphics[width=0.08\columnwidth]{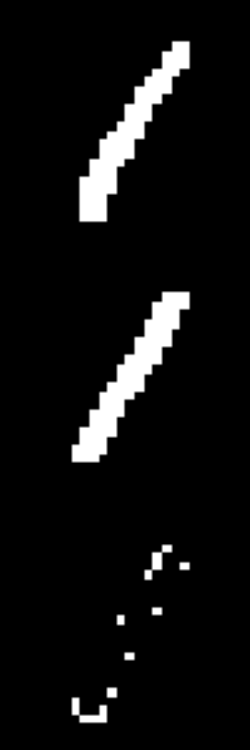}\includegraphics[width=0.08\columnwidth]{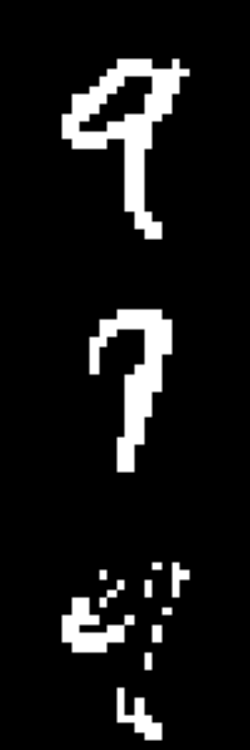}\includegraphics[width=0.08\columnwidth]{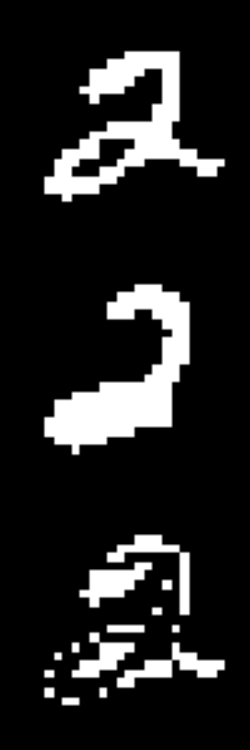}\includegraphics[width=0.08\columnwidth]{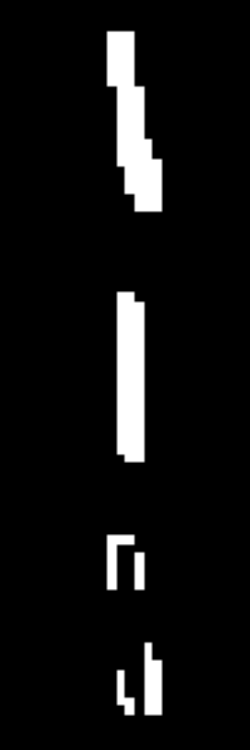}\includegraphics[width=0.08\columnwidth]{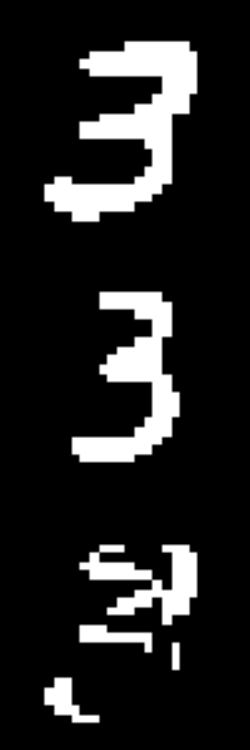}\includegraphics[width=0.08\columnwidth]{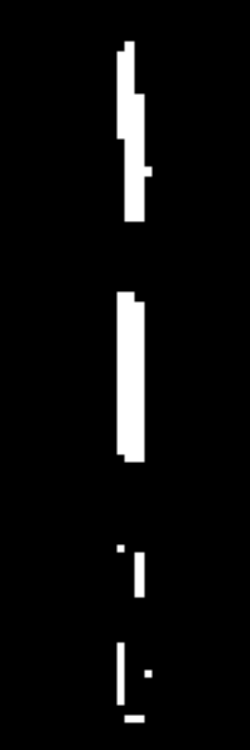}\includegraphics[width=0.08\columnwidth]{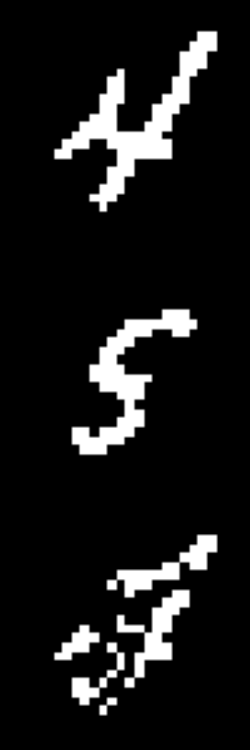}
   \caption{The top row shows the original digit, the middle row shows the most likely values of the pixels conditioned on the label, $y=1,\ldots, 20$, and the bottom row shows the remainder or residual error, $\bar X$. 
   }   \vspace{-2mm}
   \label{fig:20xbar}
\end{figure}

The results of our various compression benchmarks are shown in Table~\ref{tab:compress}. For comparison we also show results from two standard compression schemes, gzip, based on Lempel-Ziv coding~\cite{lz77}, and Huffman coding~\cite{huffman}. We take the better compression result from storing and compressing the $784 \times 50000$ data array in column-major or row-major order with these (sequence-based) compression schemes.
Note that the sieve and random representative benchmark that we described require a codebook of fixed size whose contribution is asymptotically negligible and is not included in the results. 


 
\begin{table}[htbp]
   \centering
   \topcaption{Summary of compression results. Results with a ``*'' are reported based on empirical compression results rather than Shannon bounds.} 
  {\small
   \begin{tabular}{@{} lcc @{}} 
      \toprule
      Method    & Bits per digit \\
      \midrule
      Naive & 784 \\
      Huffman*~\cite{huffman} & 376\\
      gzip*~\cite{lz77} & 328 \\
      Bitwise & 297 \\
      20 random representatives & 293 \\
      50 random representatives & 279\\
      100 random representatives & 267\\
      20 sieve representatives & 266\\
      50 sieve representatives & 252\\
      100 sieve representatives & 243\\
      \bottomrule
   \end{tabular}}
   \label{tab:compress}\vspace{-3mm}
\end{table}

\paragraph{Discussion}

First of all, sequence-based compression schemes have a serious disadvantage in this setup. Because the pixels are scrambled, to take advantage of correlations would require longer window sizes than is typical. The random compression scheme does significantly better. Despite the scrambled pixels, at least it uses the fact that the data consist of iid samples of length 784 pixels. However, the sieve leads to much better compression; for instance, 20 sieve representatives are as good as 100 random ones. 
The idea behind ``factorial codes''~\cite{barlow} is that if we can transform our data so that the variables are independent, and then (optimally) compress each variable separately, we will achieve a globally optimal compression. The compression results shown here are promising, but are not state-of-the-art. The reason is that our discovery of discrete independent components comes at a cost of increasing the cardinality of variables at each layer of the sieve. To define a more practical compression scheme, we would have to balance the trade-off between reducing dependence and controlling the size of the state space. We leave this direction for future work.

\section{Related work}\label{sec:related}

The idea of decomposing multivariate information as an underlying principle for unsupervised representation learning has been recently introduced~\cite{corex_theory, nips2014} and used in several contexts~\cite{pepke, mad_spie}. 
While bounds on $TC(X)$ were previously given, here we provided an exact decomposition. Our decomposition also introduces the idea of remainder information. While previous work required fixing the depth and number of latent factors in the representation, remainder information allows us to build up the representation incrementally, learning the depth and number of factors required as we go. Besides providing a more flexible approach to representation and structure learning, the invertibility of the information sieve makes it more naturally suited to a wider variety of tasks including lossy and lossless compression and prediction. 
Another interesting related result showed that positivity of the quantity $TC(X;Y)$ (the same quantity appearing in our bounds) implies that the $X$'s share a common ancestor in any DAG consistent with $p_X(x)$~\cite{steudelay}. A different line of work about information decomposition focuses on distinguishing synergy and redundancy~\cite{williamsbeer}, though these measures are typically impossible to estimate for high-dimensional systems. Finally, a different approach to information decomposition focuses on the geometry of the manifold of distributions defined by different models~\cite{amari}. 

Connections with ICA were discussed in Sec.~\ref{sec:ica} and the relationship to InfoMax was discussed in Sec.~\ref{sec:background}. 
The information bottleneck (IB)~\cite{tishby} is another information-theoretic optimization for constructing representations of data that has many mathematical similarities to the objective in Eq.~\ref{eq:opt}, with the main difference being that IB focuses on supervised learning while ours is an unsupervised approach. 
Recently, the IB principle was used to investigate the value of depth in the context of supervised learning~\cite{tishby_deep}. The focus here, on the other hand, is to find an information-theoretic principle that justifies and motivates deep representations for unsupervised learning. 

\section{Conclusion}\label{sec:conclusion}

We introduced the information sieve, which provides a decomposition of multivariate information for high-dimensional (discrete) data that is also computationally feasible. The extension of the sieve to continuous variables is nontrivial but appears to result in algorithms that are more robust and practical~\cite{linearsieve}.
We established here a few of the immediate implications of the sieve decomposition. First of all, we saw that a natural notion of ``remainder information'' arises and that this allows us to extract information in an incremental way.
Several distinct applications to fundamental problems in unsupervised learning were demonstrated and appear promising for in-depth exploration.
The sieve provides an exponentially faster method than the best known algorithm for discrete ICA (though without guarantees of global optimality). 
We also showed that the sieve defines both lossy and lossless compression schemes. 
Finally, the information sieve suggests a novel conceptual framework for understanding unsupervised representation learning. 
Among the many deviations from standard representation learning a few properties stand out. Representations are learned incrementally and the depth and structure emerge in a data-driven way. Representations can be evaluated information-theoretically and the decomposition allows us to separately characterize the contribution of each hidden unit in the representation. 

\section*{Acknowledgments}
GV acknowledges support from AFOSR grant FA9550-12-1-0417 and GV and AG acknowledge support from DARPA grant W911NF-12-1-0034 and IARPA grant FA8750-15-C-0071.

{ \small
\bibliographystyle{icml2016}
\bibliography{bibs/gversteeg,bibs/galstyan,bibs/gversteeg_t,bibs/versteeg} 
}

\appendix

\clearpage
\counterwithin{figure}{section}

\section*{Supplementary Material for ``The Information Sieve''}

\section{Detail of the optimization of $TC(X;Y)$}\label{sec:optimization}
We need to optimize the following objective. 
\benn
\max_{p(y|x)} \sum_i I(X_i ; Y) - I(X;Y)
\eenn
If we take the derivative of this expression (along with the constraint that $p(y|x)$ should be normalized)  and set it equal to zero, the following simple fixed point equation emerges.
\benn
p(y|x) = \frac{p(y)}{Z(x)} \prod_{i=1}^n \frac{p(x_i|y)}{p(x_i)}
\eenn
Surprisingly, optimizing this objective over possible functions has a fixed point solution with a simple form.
This leads to an iterative solution procedure that actually corresponds to a special case of the one considered in~\cite{corex_theory}. There it is shown that each iterative update of the fixed-point equation increases the objective and that we are therefore guaranteed to converge to a local optimum of the objective. In short, we consider the empirical distribution over observed samples. For each sample, we start with a random probabilistic label. Then we use these labels to estimate the marginals, $p(x_i|y)$, then we use the fixed point to re-estimate $p(y|x)$, and so on until convergence. 

Also, note that we can estimate the value of the objective in a simple way. The normalization term, $Z(x)$ is computed for each sample by just summing over the two values of $Y=y$, since $Y$ is binary. The expected logarithm of $Z$, or the free energy is an estimate of the objective ~\cite{corex_theory}. 

\paragraph{Algorithmic details}
The code implementing this optimization is included as a module in the sieve code~\cite{sieve_code}. 
The algorithm is described in Alg.~\ref{alg1}. Note we use $\delta$ as the discrete delta function. The complexity is $O(k \times N \times n)$, where $n$ is the number of variables, $N$ is the number of samples, and $k$ is the cardinality of the latent factor, $Y$. Because the solution only depends on estimation of marginals between $X_i$ and $Y$, the number of samples needed for accurate estimation is small~\cite{corex_theory}. 

\begin{algorithm}[tb]
   \caption{Optimizing $TC(X;Y)$}
   \label{alg1}
\begin{algorithmic}
   \STATE {\bfseries Input:} Data matrix, $x_i^l$, \\ $i=1,\ldots,n$ variables, $l=1,\ldots,N$ samples.
    \STATE {\bfseries Specify $k$:} Cardinality of $Y=1, \ldots, k$
     \REPEAT
   \STATE Randomly initialize $p(Y=y|X=x^l)$
      \STATE $p(Y=y) = 1/ n \sum_l p(y |x^l)$
   \FOR{$i=1$ {\bfseries to} $n$}
   \STATE $p(X_i=x_i|Y=y) = 1/N \sum_l p(Y=y|X=x^l) \delta_{x_i, x_i^l} / p(Y=y)$
   \ENDFOR
   \STATE $p(Y=y|X=x^l) = \frac{p(Y=y)}{Z(x)} \prod_{i=1}^n \frac{p(X_i = x_i^l|Y=y)}{p(X_i = x_i^l)}$
   \UNTIL{Convergence}
\end{algorithmic}
\end{algorithm}


\paragraph{Labeling test data} 
The fixed point equation above essentially gives us a simple representation of the labeling function in terms of some parameters which, in this case, just correspond to the marginal probability distributions. We simply input values of $x$ from a test set into that equation, and then round $y$ to the most likely value to generate labels. 

\paragraph{Missing data} Note that missing data is handled quite gracefully in this scenario. Imagine that some subset of the $X_i$'s are observed. Denote the subset of indices for which we have observed data on a given sample with $G$ and the subset of random variables as $x_G$. If we solved the optimization problem for this subset only, we would get a form for the solution like this: 
\benn
p(y|x_G) = \frac{p(y)}{Z(x)} \prod_{i\in G} \frac{p(x_i|y)}{p(x_i)}.
\eenn
In other words, we simply omit the contribution from unobserved variables in the product. 


\section{Proof of Theorem~\ref{incremental}}\label{sec:inc_proof}

We begin by adopting a general definition for ``representations'' and recalling a useful theorem concerning them. 
\begin{definition}
The random variables $Y \equiv Y_1,\ldots,Y_m$ constitute a \emph{representation} of $X$ if the joint distribution factorizes, $p(x,y) = \prod_{j=1}^m p(y_j|x) p(x), \forall x \in \mathcal X, \forall j \in \{1,\ldots,m\}, \forall y_j \in \mathcal Y_j$. A representation is completely defined by the domains of the variables and the conditional probability tables, $p(y_j|x)$.
\end{definition}
\begin{theorem}
\emph{Basic Decomposition of Information}~\cite{corex_theory}
\label{basic}

If $Y$ is a representation of $X$ and we define,
\begin{align}\label{eq:tcl}
TC_{L} (X; Y) &\equiv  \sum_{i=1}^n I(Y:X_i) - \sum_{j=1}^m I(Y_j:X),
\end{align}
then the following bound and decomposition holds. 
\be\label{eq:basic}
TC(X) \geq TC(X;Y) = TC(Y) + TC_L(X;Y)
\ee
\end{theorem}

\begin{theorem*} 
\emph{ Incremental Decomposition of Information}

Let $Y$ be some (deterministic) function of $X_1, \ldots, X_n$ and for each $i = 1,\ldots,n$, $\bar X_i$ is a probabilistic function of $X_i, Y$. Then the following upper and lower bounds on $TC(X)$ hold.
\be\label{eq:bounds2}
 - \sum_{i=1}^n I(\bar X_i ; Y) \leq  \\ 
 TC(X) - \left(TC(\bar X) + TC(X;Y)\right) \leq  \\
 \sum_{i=1}^n H(X_i | \bar X_i, Y)
\ee 
\end{theorem*} 
\begin{proof}
We refer to Fig.~\ref{fig:sieve}(a) for the structure of the graphical model. We set $\bar X \equiv \bar X_1,\ldots, \bar X_n, Y$ and we will write $\bar X_{1:n}$ to pick out all terms except $Y$. Note that because $Y$ is a deterministic function of $X$, we can view $\bar X_i$ as a probabilistic function of $X_i, Y$ or of $X$ (as required by Thm.~\ref{basic}). 
Applying Thm.~\ref{basic}, we have
$$ TC(X; \bar X) = TC(\bar X) + TC_L(X; \bar X).$$
On the LHS, note that $TC(X; \bar X) = TC(X) - TC(X | \bar X)$, so we can re-arrange to get 
\be\label{eq:step1}
TC(X) - (TC(\bar X) + TC(X;Y)) \\
= TC(X|\bar X) + TC_L(X;\bar X) - TC(X;Y).
\ee 
The LHS is the quantity we are trying to bound, so we focus on expanding the RHS and bounding it. 

First we expand $TC_L(X; \bar X) = \sum_{i=1}^n I(X_i ; \bar X) - \sum_{i=1}^n I(\bar X_i ; X) - I(Y;X)$. Using the chain rule for mutual information we expand the first term. 
\benn
TC_L(X; \bar X) =& \sum_{i=1}^n I(X_i ; Y) \\ &+ \sum_{i=1}^n I(X_i ; \bar X_{1:n} | Y)  \\ &- \sum_{i=1}^n I(\bar X_i ; X) - I(Y;X). 
\eenn
Rearranging, we take out a term equal to $TC(X;Y)$.
\benn
TC_L(X; \bar X) =& TC(X;Y) + \\
&\sum_{i=1}^n I(X_i ; \bar X_{1:n} | Y) - \sum_{i=1}^n I(\bar X_i ; X). 
\eenn
We use the chain rule again to write $I(X_i ; \bar X_{1:n} | Y) = I(X_i ; \bar X_i | Y) + I(X_i ; \bar X_{\tilde i} | Y \bar X_i)$, where $\bar X_{\tilde i} \equiv \bar X_1, \ldots, \bar X_n$ with $\bar X_i$ (and $Y$) excluded. 
\benn
TC_L(X; \bar X) =& TC(X;Y) + \sum_{i=1}^n ( I(X_i ; \bar X_i | Y) 
\\ &+I(X_i ; \bar X_{\tilde i} | Y \bar X_i) - I(\bar X_i ; X) ). 
\eenn
The conditional mutual information, $I(A;B|C) = I(A;BC) - I(A;C)$. We expand the first instance of CMI in the previous expression.
\benn
TC_L(X; \bar X) = & TC(X;Y) + \sum_{i=1}^n ( I( \bar X_i; X_i, Y) 
\\ & - I(\bar X_i ; Y) +I(X_i ; \bar X_{\tilde i} | Y \bar X_i)
\\ &  - I(\bar X_i ; X) ).
\eenn
Since $Y=f(X)$, the first and fourth terms cancel.
Finally, this leaves us with 
\benn TC_L(X; \bar X) = & TC(X;Y) - \sum_{i=1}^n I(\bar X_i ; Y) 
\\ & + \sum_{i=1}^n I(X_i ; \bar X_{\tilde i} | Y \bar X_i).
\eenn
Now we can replace all of this back in to Eq.~\ref{eq:step1}, noting that the $TC(X;Y)$ terms cancel.
\be\label{eq:step2}
TC(X) - (TC(\bar X) + TC(X;Y)) \\
= TC(X|\bar X) - \sum_{i=1}^n I(\bar X_i ; Y) + \sum_{i=1}^n I(X_i ; \bar X_{\tilde i} | Y \bar X_i).
\ee 
First, note that total correlation, conditional total correlation, mutual information, conditional mutual information, and entropy (for discrete variables) are non-negative. Therefore we trivially have the lower bound, $LHS \geq - \sum_{i=1}^n I(\bar X_i ; Y) $. All that remains is to find the upper bound. We drop the negative mutual information, expand the definition of $TC$ in the first line, then drop the negative of an entropy in the second line. 
\benn
LHS &\leq \sum_{i=1}^n H(X_i | \bar X) - H(X|\bar X) + \sum_{i=1}^n I(X_i ; \bar X_{\tilde i} | Y \bar X_i) \\
& \leq \sum_{i=1}^n \left(H(X_i | \bar X)  + I(X_i ; \bar X_{\tilde i} | Y \bar X_i) \right) \\
&= \sum_{i=1}^n H(X_i | \bar X_i, Y)
\eenn
The equality in the last line can be seen by just expanding all the definitions of conditional entropies and conditional mutual information. These provide the upper and lower bounds for the theorem. 
\end{proof}

\begin{figure}[htbp] 
   \centering
   \includegraphics[width=0.95\columnwidth]{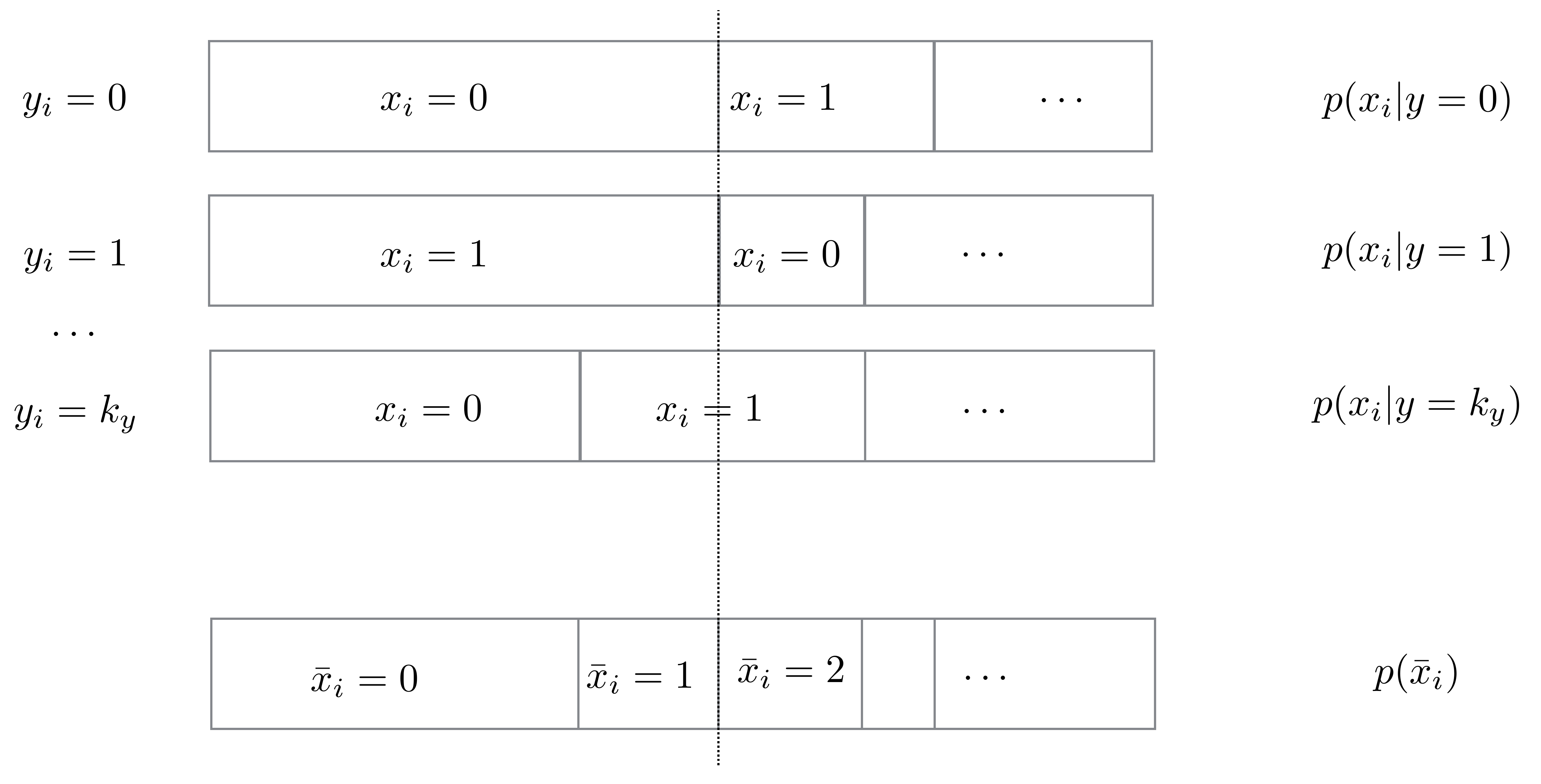} 
   \caption{An illustration of how the remainder information, $\bar x_i$, is constructed from statistics about $p(x_i,y)$. }
   \label{fig:remainder}
\end{figure}

\begin{figure*}[!htbp] 
   \centering
   \includegraphics[width=0.03\textwidth]{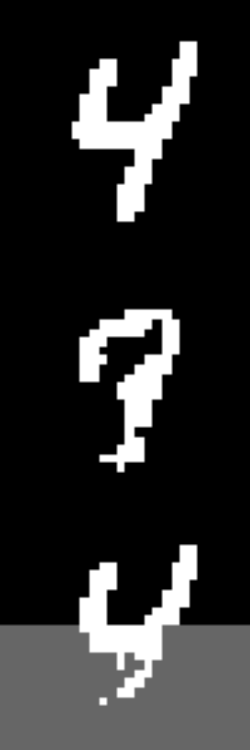}\includegraphics[width=0.03\textwidth]{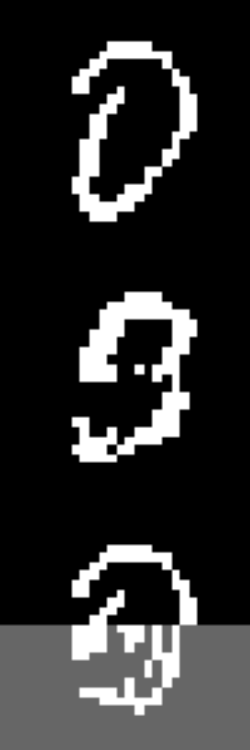}\includegraphics[width=0.03\textwidth]{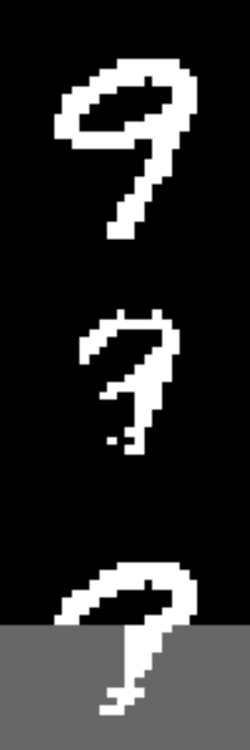}\includegraphics[width=0.03\textwidth]{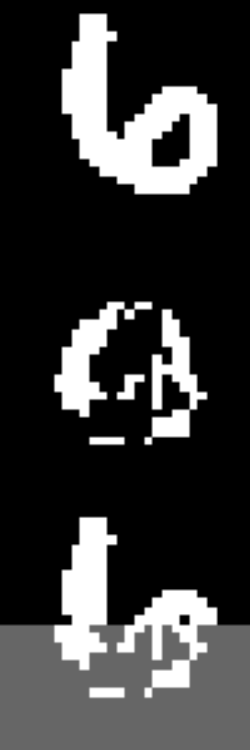}\includegraphics[width=0.03\textwidth]{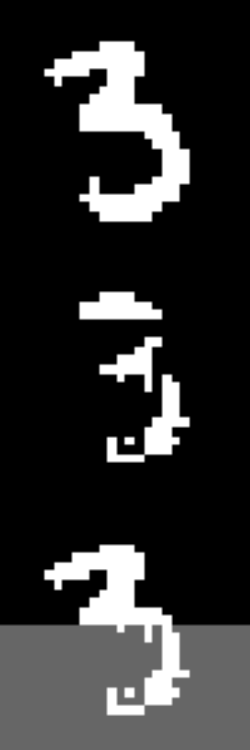}\includegraphics[width=0.03\textwidth]{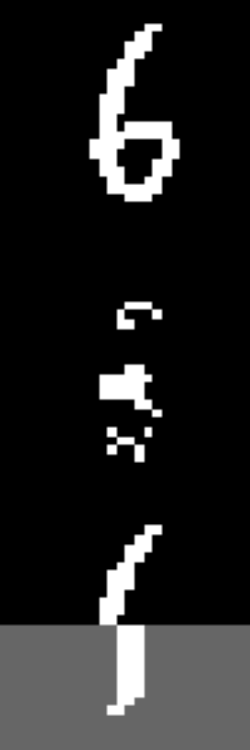}\includegraphics[width=0.03\textwidth]{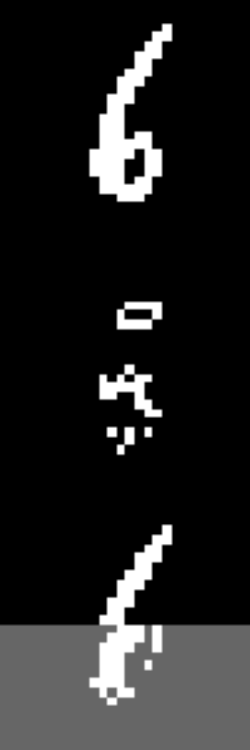}\includegraphics[width=0.03\textwidth]{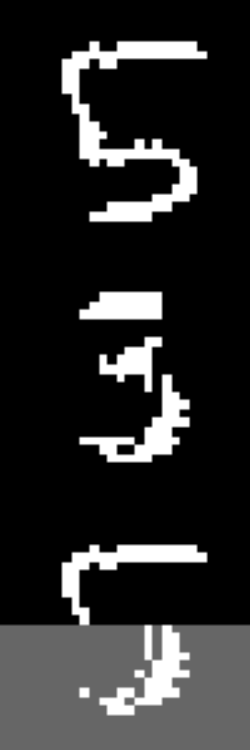}\includegraphics[width=0.03\textwidth]{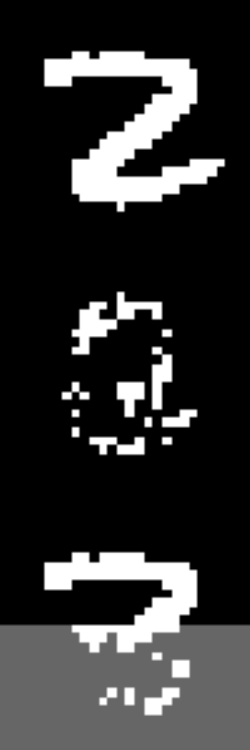}\includegraphics[width=0.03\textwidth]{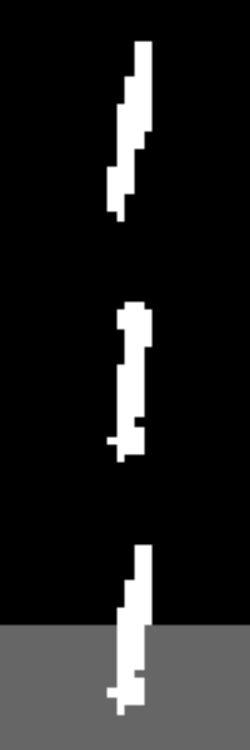}\includegraphics[width=0.03\textwidth]{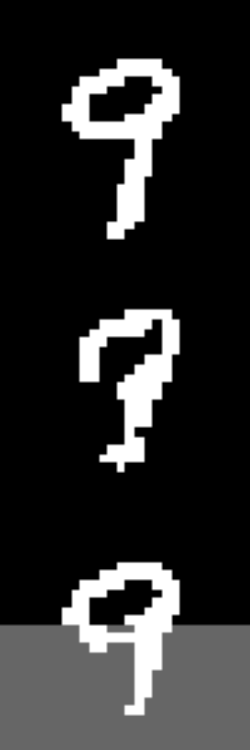}\includegraphics[width=0.03\textwidth]{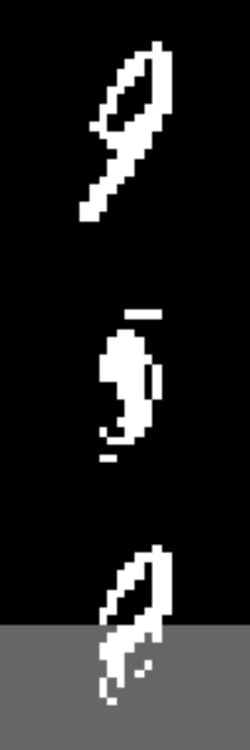}\includegraphics[width=0.03\textwidth]{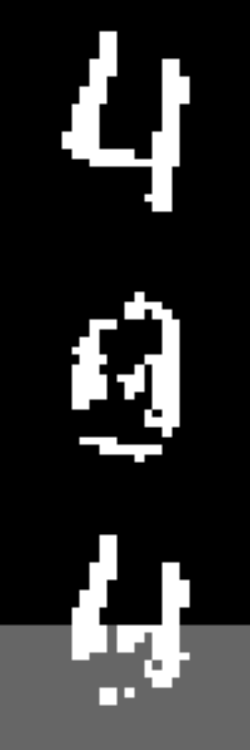}\includegraphics[width=0.03\textwidth]{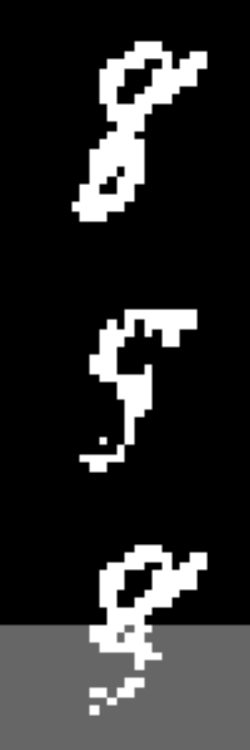}\includegraphics[width=0.03\textwidth]{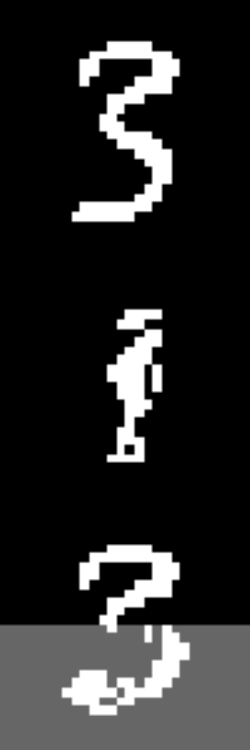}\includegraphics[width=0.03\textwidth]{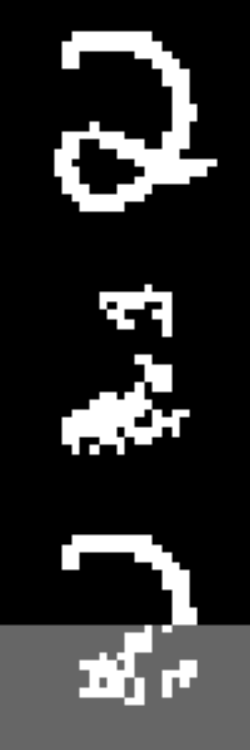}\includegraphics[width=0.03\textwidth]{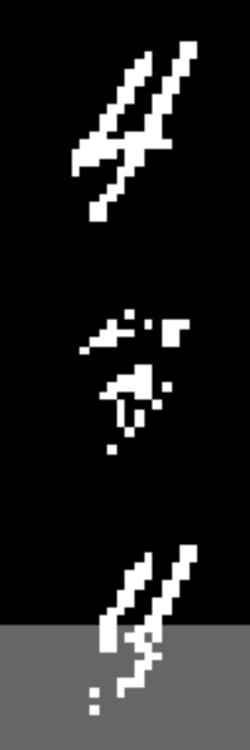}\includegraphics[width=0.03\textwidth]{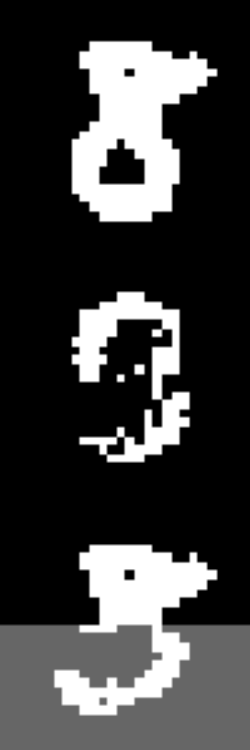}\includegraphics[width=0.03\textwidth]{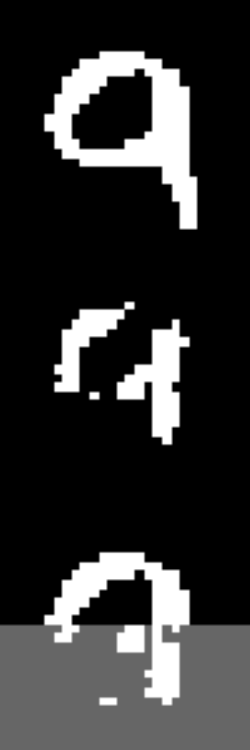}\includegraphics[width=0.03\textwidth]{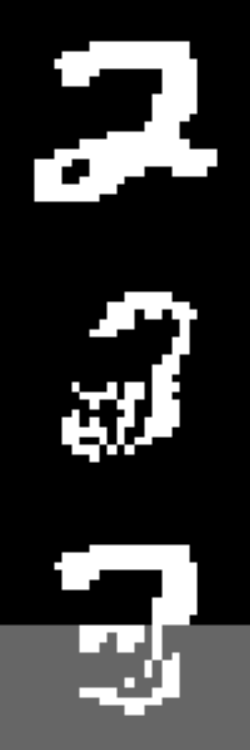}\includegraphics[width=0.03\textwidth]{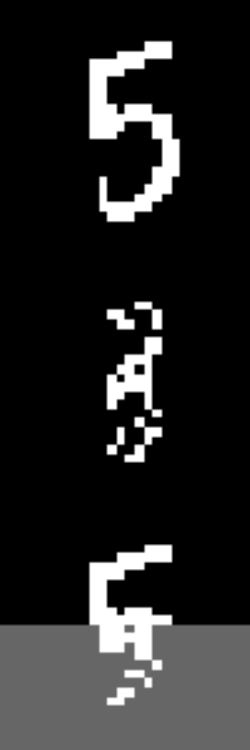}\includegraphics[width=0.03\textwidth]{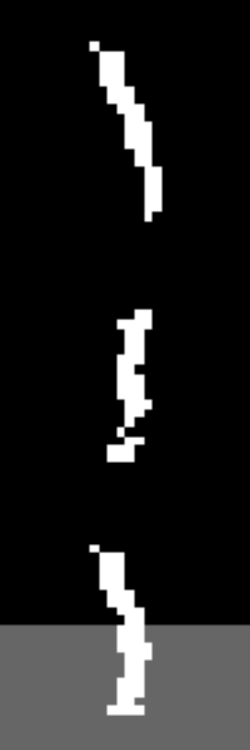}\includegraphics[width=0.03\textwidth]{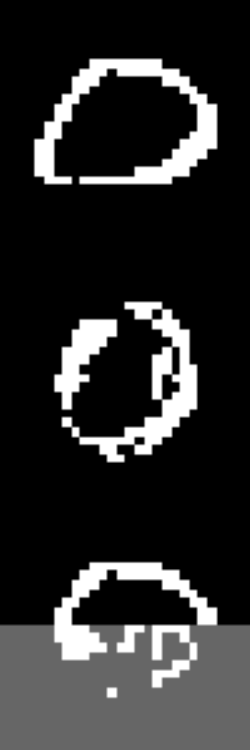}\includegraphics[width=0.03\textwidth]{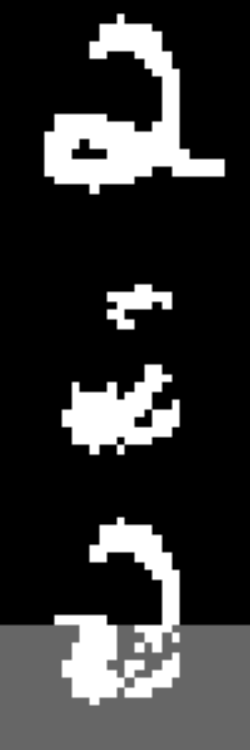}\includegraphics[width=0.03\textwidth]{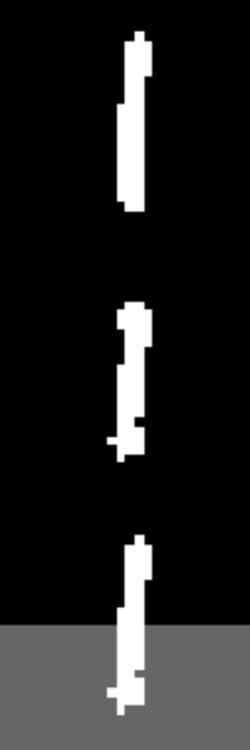}
   \caption{The same results as Fig.~\ref{fig:digits} but using samples from a test set instead of the training set.
   }   \vspace{-4mm}
   \label{fig:more_digits}
\end{figure*}

\section{An algorithm for perfect reconstruction of remainder information}\label{sec:remainder}

We will use the notation of Fig.~\ref{fig:sieve}(a) to construct remainder information for one variable in one layer of the sieve. The goal is to construct the remainder information, $\bar X_i$, as a probabilistic function of $X_i, Y$ so that we satisfy the conditions of Lemma~\ref{lemma}, $$(i) \qquad I(\bar X_i ;Y)=0 \qquad (ii) \qquad H(X_i|\bar X_i,Y) =0.$$
We need to write down a probabilistic function $p(\bar x_i | x_i, y)$ so that, for the observed statistics, $p(x_i, y)$, these conditions are satisfied. 
There are many ways to accomplish this, and we sketch out one solution here. The actual code we use to generate remainder information for results in this paper are available~\cite{sieve_code}.

We start with the picture in Fig.~\ref{fig:remainder} that visualizes the conditional probabilities $p(x_i|y)$. Note that the order of the $x_i$ for each value of $y$ can be arbitrary for this scheme to succeed. For concreteness, we sort the values of $x_i$ for each $y$ in order of descending likelihood. Next, we construct the marginal distribution, $p(\bar x_i)$. Every time we see a split in one of the histograms of $p(x_i|y)$, we introduce a corresponding split for $p(\bar x_i)$. Now, to construct $p(\bar x_i | x_i, y)$, for each $\bar x_i = q$, for each $y=j$, we find the unique value of $x_i = k(j,q)$ that is directly above the histogram for $p(\bar x_i = q)$. Then we set $p(\bar x_i = q | x_i, y) = p(\bar x_i = q) / p(x_i = k(j,q) | y=j)$. Now, marginalizing over $x_i$, $p(\bar x_i |y) = p(\bar x_i)$, ensuring that $I(\bar X_i ;Y) =0$. Visually, it can be seen that $H(X_i | \bar X_i, Y) = 0$ by picking a value of $\bar x_i$ and $y$ and noting that it picks out a unique value of $x_i$ in Fig.~\ref{fig:remainder}. 

Note that the function to construct $\bar x_i$ is probabilistic. Therefore, when we construct the remainder information at the next layer of the sieve, we have to draw $\bar x_i$ stochastically from this distribution. In the example in Sec.~\ref{sec:implement} the functions for the remainder information happened to be deterministic. In general, though, probabilistic functions inject some noise to ensure that correlations with $Y$ are forgotten at the next level of the sieve. In Sec.~\ref{sec:lossless} we point out that this scheme is detrimental for lossless compression and we point out an alternative.

\para{Controlling the cardinality of $\bar x_i$} It is easy to imagine scenarios in Fig.~\ref{fig:remainder} where the cardinality of $\bar x_i$ becomes very large. What we would like is to be able to approximately satisfy conditions (i) and (ii) while keeping the cardinality of the variables, $\bar X_i$, small (so that we can accurately estimate probabilities from samples of data). To guide intuition, consider two extreme cases. First, imagine setting $\bar x_i = 0$, regardless of $x_i,y$. This satisfies condition (i) but maximally violates (ii). The other extreme is to set $\bar x_i = x_i$. In that case, (ii) is satisfied, but $I(\bar X_i ;Y) = I(X_i;Y)$. This is only problematic if $X_i$ is related to $Y$ to begin with. If it is, and we set $\bar X_i = X_i$, then the same dependence can be extracted at the next layer as well (since we pass $X_i$ to the next layer unchanged). 

In practice we would like to find the best solution with a cardinality of fixed size. Note that this can be cast as an optimization problem where $p(\bar x_i =  | x_i, y)$ represent $\bar k \times k_x \times k_y$ variables to optimize over if those are the respective cardinalities of the variables. Then we can minimize a nonlinear objective like $\mathcal O = H(X_i | \bar X_i, Y) + I(\bar X_i ; Y)$ over these variables. While off-the-shelf solvers will certainly return local optima for this problem, the optimization is quite slow, especially if we let $k$'s get big.

For the results in this paper, instead of directly solving the optimization problem above to get a representation with cardinality of fixed size, we first construct a perfect solution without limiting the cardinality. Then we modify that solution to let either (i) or (ii) grow somewhat while reducing the cardinality of $\bar x_i$ to some target. To keep $I(\bar X_i ; Y) = 0$ while reducing the cardinality of $\bar x_i$, we just pick the $\bar x_i$ with the smallest probability and merge it with another value for $\bar x_i$.  
On the other hand, to reduce the cardinality while keeping $H(X_i | \bar X_i, Y) = 0$, we again start by finding the $\bar x_i = k$ with the lowest probability. Then we take the probability mass for $p(\bar x_i =k | x_i, y)$ for each $x_i$ and $y$ and add it to the $p(\bar x_i \neq k | x_i, y)$ that already has the highest likelihood for that $x_i,y$ combination. Note that $I(\bar X_i;Y)$ will no longer be zero after doing so. For both of these schemes (keeping (i) fixed or keeping (ii) fixed) we reduce cardinality until we achieve some target. For the results in this paper we alway picked $k_{\bar x_i} = k_{x_i} + 1$ as the target and we always used the strategy where (ii) was satisfied and we let (i) be violated. In cases where perfect remainder information is impractical due to issues of finite data, we have to define ``good remainder information'' based on how well it preserves the bounds in Thm.~\ref{incremental}. The best way to do this may depend on the application, as we saw in Sec.~\ref{sec:lossless}.

\section{More MNIST results}\label{sec:more_mnist}

Fig.~\ref{fig:more_digits} shows the same type of results as Fig.~\ref{fig:digits} but using test data that was never seen in training. Note that no labels were used in any training. 

There are several plausible to generate new, never before seen images using the sieve. Here we chose to draw the variables at the last layer of the sieve randomly and independently according to each of their marginal distributions over the training data. Then we inverted the sieve to recover hallucinated images. Some example results are shown in Fig.~\ref{fig:hallucinate}. 

\begin{figure}[!htbp] 
   \centering
   \includegraphics[width=0.6\columnwidth]{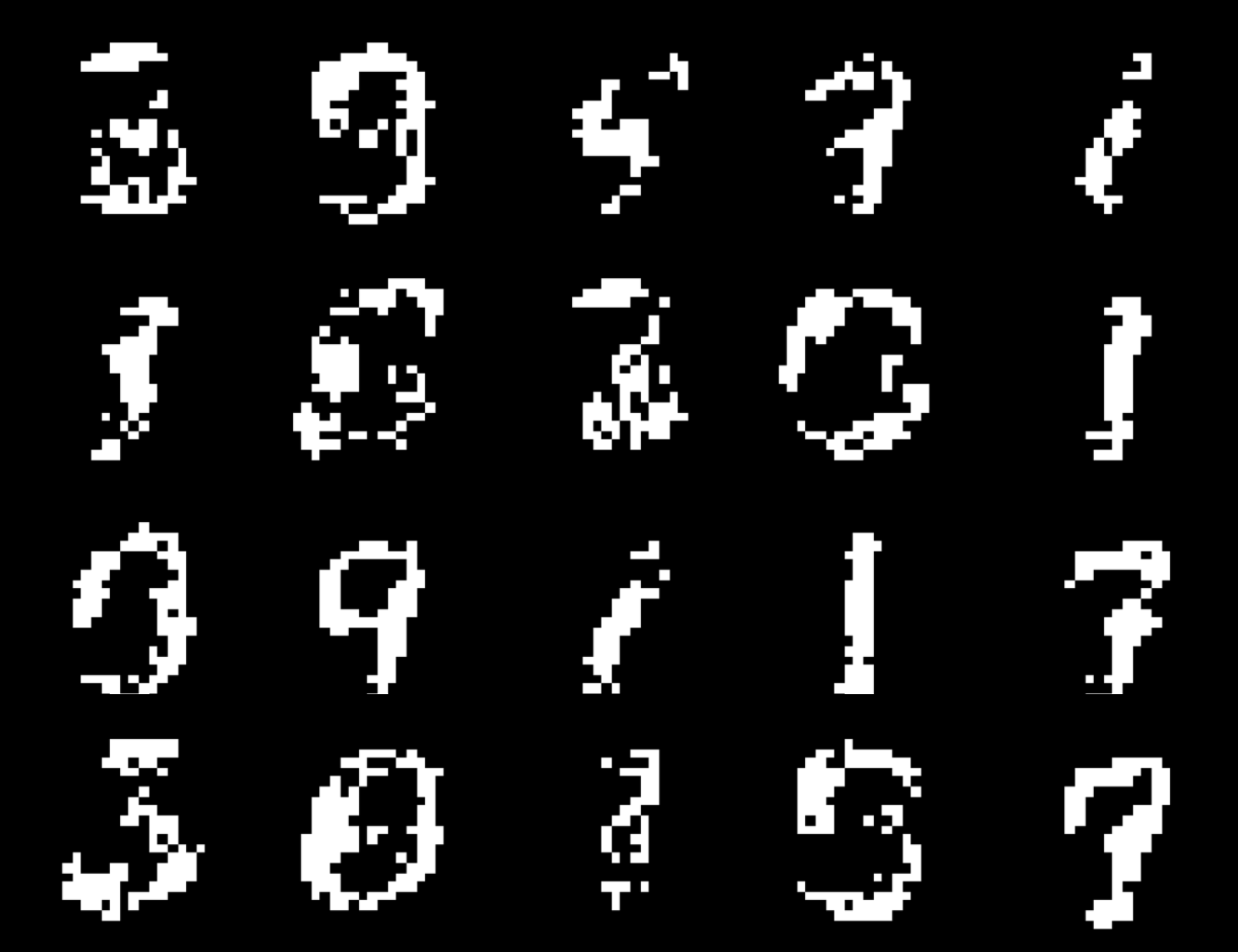} 
   \caption{An attempt to generate new images using the sieve. }
   \label{fig:hallucinate}
\end{figure}


\end{document}